\documentclass[letterpaper]{article} % DO NOT CHANGE THIS
\usepackage{aaai23} % [submission] % DO NOT CHANGE THIS
\usepackage{times}  % DO NOT CHANGE THIS
\usepackage{helvet}  % DO NOT CHANGE THIS
\usepackage{courier}  % DO NOT CHANGE THIS
\usepackage[hyphens]{url}  % DO NOT CHANGE THIS
\usepackage{graphicx} % DO NOT CHANGE THIS
\usepackage{natbib}  % DO NOT CHANGE THIS AND DO NOT ADD ANY OPTIONS TO IT
\usepackage{caption} % DO NOT CHANGE THIS AND DO NOT ADD ANY OPTIONS TO IT
\usepackage{algorithm}
\usepackage{multirow}%
\usepackage{amsmath,amssymb,amsfonts}%
\usepackage{amsthm}%
\usepackage{mathrsfs}%
\usepackage{xcolor}%
\usepackage{textcomp}%
\usepackage{manyfoot}%
\usepackage{booktabs}%
\usepackage{listings}%
\usepackage{tabularx}
\usepackage{subfig}

\theoremstyle{thmstyleone}%
\newtheorem{theorem}{Theorem}%  meant for continuous numbers
\theoremstyle{thmstyletwo}%

\theoremstyle{thmstylethree}%
\usepackage{newfloat}
\usepackage{listings}
\DeclareCaptionStyle{ruled}{labelfont=normalfont,labelsep=colon,strut=off} % DO NOT CHANGE THIS
\floatstyle{ruled}
\newfloat{listing}{tb}{lst}{}
\floatname{listing}{Listing}
\title{Quantifying Deep Learning Model Uncertainty in Conformal Prediction}
\author {
    Hamed Karimi\textsuperscript{\rm 1},
    Reza Samavi\textsuperscript{\rm 1,2}
}
\affiliations {
    \textsuperscript{\rm 1} Department of Electrical, Computer, and Biomedical Engineering\\
    Toronto Metropolitan University, Toronto, ON, Canada\\
    \textsuperscript{\rm 2} Vector Institute, Toronto, ON, Canada\\
    hamed.karimi@torontomu.ca, samavi@torontomu.ca
}
\begin{document}

\maketitle

\begin{abstract}
Precise estimation of predictive uncertainty in deep neural networks is a critical requirement for reliable decision-making in machine learning and statistical modeling, particularly in the context of medical AI. Conformal Prediction (CP) has emerged as a promising framework for representing the model uncertainty by providing well-calibrated confidence levels for individual predictions. However, the quantification of model uncertainty in conformal prediction remains an active research area,
%due to its challenges and limitations, 
yet to be fully addressed. In this paper, we explore state-of-the-art CP methodologies and their theoretical foundations. We propose a probabilistic approach in quantifying the model uncertainty derived from the produced prediction sets in conformal prediction and provide certified boundaries for the computed uncertainty. By doing so, we allow model uncertainty measured by CP to be compared by other uncertainty quantification methods such as Bayesian (e.g., MC-Dropout and DeepEnsemble) and Evidential approaches.  

%enhances the model reliability by providing a formal guarantee in the model uncertainty boundaries. Furthermore, the method 
%facilitates the capability of CP to be compared with the existing uncertainty quantification approaches.
%Real-world applications can highly benefit from decision-making under uncertainty and risk assessment. This research advances model uncertainty quantification in conformal prediction, providing valuable approach for professionals, particularly, medical practitioners. Future research opportunities are identified for further development in this field.
\end{abstract}

\section{Introduction}
\label{intro}
Accurate estimation of predictive uncertainty plays a crucial role in high-stakes real-world applications, particularly in the field of medical AI, where precise and reliable classification of diseases and conditions is paramount. Machine learning models have demonstrated their potential in aiding medical professionals with accurate diagnosis and treatment decisions. However, relying solely on point predictions without considering the associated uncertainty can lead to erroneous conclusions and suboptimal patient care.

To illustrate the significance of model uncertainty quantification in medical AI, let us consider a classification task involving the identification of different types of skin lesions based on diagnostic images. The machine learning model is trained on a large dataset of annotated skin lesion images, along with corresponding clinical information. The goal of the model is to classify new, unseen images into specific categories, such as malignant melanoma, benign nevi, or basal cell carcinoma.
Suppose the model predicts a given skin lesion as malignant melanoma, indicating a high probability of malignancy. Without an understanding of the associated uncertainty, medical professionals may proceed with aggressive treatment or surgical intervention, potentially subjecting patients to unnecessary procedures. There are inherent uncertainties in the prediction, stemming from various sources, such as variations in imaging quality, complex morphological features, or overlapping characteristics between different lesion types.
Therefore, for a prediction to be acceptable, in addition to the ability to achieve high predictive accuracy, it is also crucial to have a measure of the predictive uncertainty. 
%in order to calibrate the model's probabilities in Convolutional image classifiers in spite of the ability to achieve high predictive accuracy. 

Although there are popular approaches to quantify the model predictive uncertainty, e.g. Bayesian methods such as MC-Dropout~\cite{dropout} and DeepEnsemble~\cite{ensemble}, and Evidential approaches~\cite{evid1,evid2,evid_risk}, the lack of formal guarantees is a major limitation in the state-of-the-art methods of uncertainty quantification. To resolve this issue, Conformal Prediction (CP) or Conformal Inference~\citep{cp,inductive_cp} provides a compelling framework as a post-processing technique to address this challenge by offering a reliable indicator of uncertainty. 
Rather than providing a single deterministic prediction, CP constructs a finite \textit{Prediction Set} or \textit{Uncertainty Set} that encompasses a plausible subset of class labels for a given unseen input data point in any pretrained classifier. This prediction set reflects the inherent uncertainty associated with the model's predictions. In addition to a point estimation of the most likely predictive probability as a measure of confidence, the size of a prediction set is considered as an indicator of the model uncertainty in classifying a new data point. Larger prediction sets indicate higher model uncertainty associated with the input data. However, in case of using the prediction set size as an uncertainty indicator, the measure is not scaled to be compared with other state-of-the-art uncertainty quantification methods.
%such as predictive entropy in Bayesian methods~\citep{dropout,ensemble} or evidential uncertainty~\citep{evid1,evid_risk}. 

Returning to our medical AI example, instead of a definite prediction of malignant melanoma (a single true label), CP offers a prediction set indicating the most likely class labels of the skin lesion with the respected probabilities, while providing a guarantee that the true label is a member of this set with a high probability. This additional information enables medical professionals to make more informed decisions, considering the potential risks and uncertainties associated with the model's predictions. CP is also fast, computationally efficient, and generally applicable to every dataset (arbitrary data distribution) and classification model~\citep{distribution_cp}.
Nevertheless, the quantification of model uncertainty in conformal prediction for classification tasks remains an active research area, with several challenges and limitations. To the best of our knowledge, there is no scaled and reliable quantification of model uncertainty achieved by CP method. The uncertainty quantification is highly crucial when there is an intent to represent the amount of model uncertainty or perform comparative evaluations. Thus, we aim to propose a novel approach to quantify the model uncertainty based on the produced prediction sets and improve the reliability and accuracy of uncertainty estimation.

In this paper, we are making the following two contributions in the context of conformal prediction: (1) we investigate existing methodologies for model uncertainty estimation within conformal prediction for classification tasks, and analyze their strengths and limitations, (2) we propose a novel technique for uncertainty quantification of CP, aiming to facilitate the comparative evaluations between CP-based methods and other state-of-the-art uncertainty estimation methods. 
We use the formal guarantee of true label coverage~\cite{cp_aps} in the prediction set to devise a solid probabilistic theory along with certified boundaries of the model uncertainty. The proposed quantification method can enhance the accuracy and reliability of uncertainty estimation in real-world applications. By advancing the field of model uncertainty quantification in conformal prediction for classification tasks, this research endeavors to equip medical professionals with a more reliable and guaranteed method to quantify the model uncertainty and support their decision-making process.

\section{Background on Conformal Prediction}
\label{cp_background}
% Conformal prediction is a post-processing technique that generally takes an unseen data point and produces the corresponding prediction set of the most likely labels. 
The validity of CP method depends on the assumption that the calibration data points are iid and exchangeable, i.e. data points are selected independently from the same distribution. In conformal procedure for a pretrained model, we use \textit{split conformal prediction} as a method of splitting the test (unseen) data into calibration data and validation data. The set of calibration data or holdout set is a small amount of additional unseen data with a size of around 500 to 1000 data points. Moreover, an arbitrary conformal score function is defined to represent a measure of discrepancy between model predictive outcomes and true labels which is used to compare an unseen data point in the validation set with those in the calibration set. CP uses the calibration data and the conformal score function to generate the prediction sets. Note that CP is not restricted to a specific conformal score function and classification/regression tasks. In the following, we discuss how to apply conformal prediction to a pretrained model.
% The produced prediction set contains the true label with a user-specified probability as a formal coverage guarantee. 

In conformal prediction, we generally take any heuristic notion of model uncertainty associated with the input data point in any data distribution and any model, and transform this uncertainty to a rigorous form. To this aim, we generally consider an unseen data point $x \in \mathcal{X}_{cal}$ 
% with the true label $y_{cal} \in \mathcal{Y}=\{1,2,\cdots,K\}$ 
from a small set of i.i.d. (independent and identically distributed) data as calibration set, and the corresponding model output $y \in\mathcal{Y}=\{1,2,\cdots,K\}$. To construct the prediction set as a rigorous uncertainty estimation, we require to:
\begin{enumerate}
    \item Identify a heuristic notion of uncertainty using the pretrained model.
    \item Devise a conformal score function $s(x,y) \in \mathbb{R}$. Larger value of the score function indicates higher model uncertainty.
    \item Compute $\widehat{q}$ as the $\frac{\lceil(n+1)(1-\delta)\rceil}{n}$ quantile of the calibration scores $\{s_i=(x_i,y_i)\}_{i=1}^{n}$ using the coverage error level $\delta$ and unseen calibration data with the size $n$.
    \item Use this quantile $\widehat{q}$ to form the prediction sets for testing data $x_{val}\in\mathcal{X}_{val}$ as,
    \begin{equation}
    \mathcal{C}(x_{val},\widehat{q}) = \{y: s(x_{val}, y) \leq \widehat{q}\}\ .
    \label{cp_set}
    \end{equation}
\end{enumerate}
The prediction set $\mathcal{C}(x_{val},\widehat{q})$ is a subset of all possible $K$ labels that the model finds plausible for the image $x_{val}$.
%into them under uncertainty. 
The prediction set contains a relatively small number of labels and is guaranteed to include the true label with a user-specified probability (e.g., 90\%) in a relatively small prediction set of labels~\citep{distribution_cp}. This inclusion probability is expressed as an arbitrary coverage error level $\delta$ with which we set the probability that the prediction set contains the true label, to $1-\delta$. This validity attribute is called \textit{marginal coverage}, since the probability is averaged over the stochasticity in the unseen data points. CP satisfies the true label coverage property as its validity criterion~\citep{cp2} which will be formally discussed in Theorem~\ref{cp_coverage_theorem}.

%The prediction sets are statistically and provably guaranteed to include the true label with a high user-specified probability (e.g., 90\%) as a formal coverage guarantee, in a relatively small prediction set of labels~\citep{distribution_cp}. We can specify an arbitrary coverage error level $\delta$ with which we set the probability that the prediction set contains the true label, to almost exactly $1-\delta$. This validity attribute is called \textit{marginal coverage}, since the probability is averaged over the stochasticity in the unseen data points. CP satisfies the true label coverage property as its validity criterion~\citep{cp2} which will be formally discussed in Theorem~\ref{cp_coverage_theorem}.

According to split conformal prediction method, we formally have $\mathcal{X}_{cal}=\{(x_i,y_i)\}_{i=1}^n$ as a small set of i.i.d. and unseen calibration data,
%$\mathcal{X}_{cal}$ 
in which $x_i \in \mathbb{R}^d$ is a feature vector of size $d$ as $i$th input data point, and $y_i \in \mathcal{Y}=\{1,2,\cdots,K\}$ is the corresponding true label out of $K$ possible target labels. Moreover, $(x_{val},y_{val})\in \mathcal{X}_{val}$ is a validation data point which is unseen during the training process. After computing the conformal scores associated with the calibration data, $\widehat{q}$ is obtained as the $1-\delta$ quantile of conformal scores in which $\delta$ is a user-specified error level of true label coverage. Consider $\mathcal{C}(x_{val},\widehat{q}): \mathbb{R}^d \times \mathbb{R} \rightarrow 2^{\mathcal{Y}}$ is a function that takes an unseen input data vector and $1-\delta$ quantile of their corresponding conformal scores, and then produces a prediction set containing a subset of possible labels.  
% The second argument u is included to allow for randomized procedures; let U1, . . . , Un be i.i.d. uniform [0, 1] random variables that will serve as the second argument for each data point. Suppose that the sets are indexed by τ such that they are nested, meaning larger values of τ lead to larger sets:
% C(x, u, τ1) ⊆ C(x, u, τ2) if τ1 ≤ τ2. (2)
% To find a function that will achieve 1 − α coverage on test data, we select the smallest τ that gives at least 1 − α coverage on the conformal calibration set, with a slight correction to account for the
% finite sample size:
% τˆccal = inf τ : |{i : Yi ∈ C(Xi, Ui, τ )}| n ≥ d(n + 1)(1 − α) e n
% The set function C(x, u, τ ) with this data-driven choice of τ is guaranteed to have correct finite sample coverage on a fresh test observation, as stated formally next.
Assuming the data exchangeability, this prediction set is statistically certified to marginally cover the true label associated with a validation data point with the probability of at least $1-\delta$.
%as stated in the following theorem:

\begin{theorem}[Conformal Coverage Guarantee~\citep{cp,inductive_cp}]
Consider $\{(x_i,y_i) \in \mathcal{X}_{cal}\}_{i=1}^n$ and $(x_{val},y_{val}) \in \mathcal{X}_{val}$ are i.i.d. and unseen data as $n$ calibration data points and a validation data point, respectively. Let $\delta$ be the user-chosen coverage error level, $\widehat{q}$ is the $1-\delta$ quantile of calibration conformal scores, and $\mathcal{C}(x_{val},\widehat{q}) \subseteq \mathcal{Y}$ be the function of producing prediction set. If $\mathcal{C}(x_{val},\widehat{q})$ gradually grows to include all possible labels in $\mathcal{Y}$ when having large enough $\widehat{q}$, then, the probability of the true label being covered in the prediction set is guaranteed in the following bounds: 
\begin{equation}
    1-\delta \leq \mathcal{P}(y_{val} \in \mathcal{C}(x_{val},\widehat{q})) \leq 1-\delta+\frac{1}{n+1}\ .
\label{cp_coverage_eq}
\end{equation}
\label{cp_coverage_theorem}
\end{theorem} 
\noindent The proof and the related conditions of this theorem are available in~\citep{cp,distribution_cp}.

When constructing valid prediction sets, three distinct properties are required to be satisfied: (1) the \textit{marginal coverage property} of the true label that guarantees the prediction set includes the true label with the probability of at least $1-\delta$ based on Equation~\ref{cp_coverage_eq}, (2) the \textit{set size property} to reflect the desirability of a smaller size for the  prediction set, and (3) the \textit{adaptivity property} that necessitates the set size for unseen data is modified to represent instance-wise model uncertainty, i.e., the set size is smaller when the model encounters easier test data rather than the inherently harder ones. Note that the difficulty of a test data point is based on the rank of its true label in the sorted set of outcome probabilities. These properties affect each other; for example, the set size property tries to make the sets smaller, while the adaptivity property tries to make the sets larger for harder data points when the model is uncertain, or choosing the fixed-size sets may satisfy the coverage property, but without adaptivity.

% The produced prediction set is also valid since it satisfies the marginal coverage property based on Equation~\ref{cp_coverage_eq}.
As an example, to construct the prediction sets, we require to have a pretrained model $\mathcal{M}_{\Theta}$ with the parameter set $\Theta$ accompanied by its heuristic notion of uncertainty to form an arbitrary conformal score function, e.g., one minus the softmax probability $\mathcal{M}_{\Theta}(x,y_{true})$ associated with true label $y_{true}$ given the input data point $x$.  
However, the softmax probabilities are unreliable due to being overconfident or underconfident~\cite{calib1, calib2}. Thus, split conformal prediction method offers using a small calibration set of unseen data (not seen during the training process) to apply conformal score function and statistically achieve coverage guarantee. 
% The calibration data include around $n\approx 1000$ unseen data points  \rs {Hamed please, you repeated 6 times about the size of callibration set, please try to be unredundent from the get go. please check the entire paper, this should be said only once or if it's absoloutely necessary twice} . 
For the calibration data, we compute the aforementioned conformal scores which is higher when the model is more uncertain in the prediction. Then, we compute $\widehat{q}$ as $1-\delta$ quantile of the calibration conformal scores. For instance, if $\delta=0.1$ is set for calibration data, at least $90\%$ of softmax probabilities associated with the true labels are certified to be above the $1-\widehat{q}$. Eventually, for each validation data point $x_{val} \in \mathcal{X}_{val}$ (unseen testing data points), we include all the labels with the softmax probability above $1-\widehat{q}$ into the prediction set $\mathcal{C}(x_{val},\widehat{q})$ as,
\begin{equation}
\mathcal{C}(x_{val},\widehat{q}) = \{y: \mathcal{M}_{\Theta}(x_{val}) \geq 1-\widehat{q}\}\ .
\label{cp_set_example}
\end{equation}
Therefore, the softmax probability associated with the true label is statistically certified to be above $1-\widehat{q}$ with the probability of $90\%$, so that the marginal true label coverage is guaranteed based on Equation~\ref{cp_coverage_eq} in Theorem~\ref{cp_coverage_theorem}.

Considering the size of the prediction set as the only uncertainty measure is not reliable due to the probabilistic nature of conformal method, i.e., the existence of the true label in the prediction sets is stochastic w.r.t. the coverage error level $\delta$. %to be able to compute the associated uncertainty. 
In the following section, 
%based on the produced prediction sets in conformal prediction, 
we will propose a new quantification approach for the model uncertainty which considers the probabilistic existence of true labels.
%, yet it is simple to be measured.

\section{Related Work}
\label{related_work}
A naive approach to generate prediction sets for test data is to use a score function, e.g., softmax function, and include labels from the most likely to the least likely probabilities until their cumulative summation exceeds the threshold $1-\delta$. In this approach, the true label coverage cannot be guaranteed since the output probabilities are overconfident and uncalibrated~\citep{calib2}. Furthermore, the lower probabilities in image classifiers are significantly miscalibrated which gives rise to larger prediction sets that may misrepresent the model uncertainty. There are also a few methods to generate prediction sets, but not based on conformal prediction~\citep{non_cp1,non_cp2}. However, these methods do not have finite marginal coverage guarantees as described in Theorem~\ref{cp_coverage_theorem}.
% can be used as input to a conformal procedure to potentially improve performance.

The coverage guarantee can be achieved using a new threshold and calibration data samples as holdout set. In this regard, Romano et al.~\citep{cp_aps} proposed a method to make CP more stable in the presence of noisy small probability estimates in image classification. The authors developed a conformal method called \textit{Adaptive Prediction Set (APS)} to provide marginal coverage of true label in the prediction set which is also fully adaptive to complex data distributions using a novel conformity score, particularly for classification tasks. For example, with $\delta=0.1$, if selecting prediction sets that contain $0.85$ estimated probability can achieve $90\%$ coverage on the calibration data, APS will utilize the threshold $0.85$ to include labels in the prediction sets. However, APS still produces large prediction sets which cannot precisely represent the model uncertainty.

To mitigate the large set size, the authors in~\citep{cp_raps} introduced a regularization technique called Regularized Adaptive Prediction Sets (RAPS) to relax the impact of the noisy probability estimates which yield to significantly smaller and more stable prediction sets. RAPS modifies APS algorithm by penalizing the small conformity scores associated with the unlikely labels after Platt scaling~\citep{platt}. RAPS regularizes the APS method, therefore, RAPS acts exactly as same as APS when setting the regularization parameter to 0. Both APS and RAPS methods are always certified to satisfy the marginal coverage in Equation~\ref{cp_coverage_eq} regardless of model, architecture, and dataset. 
% RAPS also performs better rather than choosing a fixed-size set. 
Both methods also require negligible computational complexity in both finding the appropriate threshold using the calibration data with the size of $n\approx 1000$ and inference phase. However, RAPS could outperform the state-of-the-art APS by achieving marginal coverage of true labels with significantly smaller prediction sets. Thus, RAPS can produce adaptive but smaller prediction sets as an estimation of the model uncertainty given unseen image data samples.

% Consider a procedure that outputs a predictive set for each observation, and further suppose that this procedure has a tuning parameter τ that controls the size of the sets (In RAPS, τ is the cumulative sum of the sorted, penalized classifier scores). We take a small independent conformal calibration set of data, and then choose the tuning parameter τ such that the predictive sets are large enough to achieve 1 − α coverage on this set. This calibration step yields a choice of τ , and the resulting set is formally guaranteed to have coverage 1 − α on a future test point from the same distribution;
\section{Uncertainty Quantification in Prediction Sets}
\label{cp_uncer}
%\rs {so much redundancy, this paragraph has been said at least 5 times!}In conformal procedure, the final result is a prediction set that contains a subset of all possible predictive labels given an unseen data point in a pretrained model. This prediction set is normally considered as an indicator of uncertainty such that larger set size represents higher model uncertainty.  

Following Theorem~\ref{cp_coverage_theorem}, consider $\delta$ as the error level of true label coverage, $\widehat{q}$ as the computed $1-\delta$ quantile of conformal scores over calibration data with size $n$, and $\mathcal{C}(x_{val},\widehat{q}):\mathbb{R}^d \times\mathbb{R}\rightarrow2^{\mathcal{Y}}$ as the prediction set function given the unseen validation data point $(x_{val},y_{val}) \in \mathcal{X}_{val}$. The result of the function is a prediction set associated with $x_{val}$ with the size $|\mathcal{C}(x_{val},\widehat{q})|=m \geq 0$ and the maximum size of the number of all possible target labels, i.e., $m \leq |\mathcal{Y}|=K$. The true label $y_{val}$ is included in $\mathcal{C}(x_{val},\widehat{q})$ with some probabilistic boundaries. Thus, the model uncertainty $U_{\mathcal{C}}(x_{val})$ associated with the validation data point $x_{val}$ based on the corresponding prediction set $\mathcal{C}(x_{val},\widehat{q})$ can be quantified based on the following theorem:

\begin{theorem}[Conformal Uncertainty Quantification]
Suppose an unseen validation data point $(x_{val},y_{val}) \in \mathcal{X}_{val}$ is fed to a pretrained classifier with $K$ possible target labels. Let $\delta$ be the coverage error level of the true label $y_{val}$, and $\mathcal{C}(x_{val})$ be the corresponding prediction set of size $m \in \mathbb{Z}^{[0,K]}$ achieved by $1-\delta$ quantile of calibration data with size $n$. Then, the conformal model uncertainty $U_{\mathcal{C}}(x_{val})$ associated with $x_{val}$ is quantified to be $0 \leq U_{\mathcal{C}}(x_{val}) \leq 1$, and guaranteed in the following marginal lower bound $\mathcal{L}_{\mathcal{C}}$ and upper bound $\mathcal{H}_{\mathcal{C}}$ as:
\begin{itemize}
\item if $m=0$, then:
\begin{equation}
    U_{\mathcal{C}}(x_{val}) = 1\ ,
\end{equation}
\item and if $0<m \leq K$, then:
\begin{align}
\begin{split}
    &U_{\mathcal{C}}(x_{val}) \geq \widehat{u}_{\mathcal{C}}(1-\delta)+\delta-\frac{1}{n+1}=\mathcal{L}_{\mathcal{C}} \quad \text{and}\\ 
    &U_{\mathcal{C}}(x_{val}) \leq \min(\mathcal{H}_{\mathcal{C}},1) \\
    &\text{s.t.} \quad \mathcal{H}_{\mathcal{C}} = \widehat{u}_{\mathcal{C}}(\frac{n+2}{n+1})+\delta(1-\widehat{u}_{\mathcal{C}})\ ,
\end{split}
\end{align}
% \item and if $m=K$, then:
% \begin{equation}
%     1-\frac{1}{n+1} \leq U_{\mathcal{C}}(x_{val}) \leq 1\ .
% \end{equation}
\end{itemize}
where $\widehat{u}_{\mathcal{C}}$ is Pure Model Uncertainty and computed as,
\begin{equation}
    \widehat{u}_{\mathcal{C}} = \frac{m+\delta-1}{K}\ .
\end{equation}
\label{cp_uncer_theorem}
\end{theorem}

\begin{proof}
     In CP, the size of the prediction set (model's outcome) is an indicator of the total model uncertainty denoted by $u_{\mathcal{C}}$ which grows by increasing the size of the prediction set. The size of the prediction set denoted by $m$ is an integer restricted between 0 and $K$, i.e., $0\leq m \leq K$. Thus, $u_{\mathcal{C}}$ is defined as a probability over the size of the produced prediction set and computed as, 
    \begin{equation}
    u_{\mathcal{C}} = \frac{m}{K}\ ,    
    \end{equation}
     where $K$ is the number of all possible target labels. Now, we define \emph{pure model uncertainty} as our baseline uncertainty by subtracting the probability of the only certain and desired case from the total model uncertainty $u_{\mathcal{C}}$ that is when the model produces a prediction set containing only one class label (out of $K$ possible labels) which is the true label with the probability of at least $1-\delta$ according to Theorem~\ref{cp_coverage_theorem}. We compute the pure model uncertainty denoted by $\widehat{u}_{\mathcal{C}}$ as,
    \begin{equation}
    \widehat{u}_{\mathcal{C}} = u_{\mathcal{C}} - \frac{1}{K}(1-\delta) = \frac{m+\delta-1}{K}\ ,
    \label{uncer_base}
    \end{equation}
    where $\delta$ is the coverage error level of the true label. Then, we have our heuristic notion of uncertainty as the pure model uncertainty $\widehat{u}_{\mathcal{C}} \in \mathbb{R}^{[0,1]}$ which is associated with the produced prediction set $\mathcal{C}(x_{val})$ given $x_{val}$, and scaled to be used as a baseline uncertainty to quantify the conformal model uncertainty. Note that this heuristic notion of uncertainty is arbitrarily devised and can be replaced by any other heuristic and reasonable uncertainty quantification as a measure of baseline model uncertainty.
    
    Theorem~\ref{cp_uncer_theorem} has two distinct cases with respect to $m$ as the size of prediction set $\mathcal{C}(x_{val})$: 
    If $m=0$, the model is fully uncertain that could not select any target label to include into the prediction set based on the computed $\widehat{q}$. Thus, although in this case, $\widehat{u}_{\mathcal{C}} = \frac{\delta-1}{K} \leq 0$, this zero-size prediction set is treated as a special case and interpreted as the maximum model uncertainty which yields to $U_{\mathcal{C}}(x_{val}) = 1$.
    
    In the general case of $0 < m \leq K$, we have two distinct probabilistic events, the true label is either included in the prediction set denoted by $P_1$, i.e., $P_1: y_{val} \in \mathcal{C}(x_{val})$, or not included in the prediction set denoted by $P_0$, i.e., $P_0: y_{val} \notin \mathcal{C}(x_{val})$. These two inclusion events $P_1$ and $P_0$ are mutually exclusive, i.e., disjoint events, so that only one of the events can happen at the same time. We can compute the model uncertainty as the probability of the model being uncertain denoted by $P_{u}$ when either $P_1$ or $P_0$ holds as,
    \begin{align}
    \begin{split}
        U_{\mathcal{C}}(x_{val}) &= \mathcal{P}(P_{u} \land (P_1 \lor P_0)) \\
        &= \mathcal{P}((P_{u} \land P_1) \lor (P_{u} \land P_0))\ .
    \end{split}
    \end{align}
    As the two events $P_1$ and $P_0$ are disjoint, their corresponding joint events $P_{u} \land P_1$ and $P_{u} \land P_0$ are also mutually exclusive. Therefore, we have:
    \begin{align}
    \begin{split}
        U_{\mathcal{C}}(x_{val}) &= \mathcal{P}((P_{u} \land P_1) \lor (P_{u} \land P_0)) \\
        &= \mathcal{P}(P_{u} \land P_1) + \mathcal{P}(P_{u} \land P_0)\ .
    \end{split}
    \end{align}
    Each joint event can be written based on its own conditional probability of the model being uncertain (i.e., $P_{u}$) given the inclusion of the true label in the prediction set (i.e., $P_1$ or $P_0$) as,
    \begin{align}
    \begin{split}
    % \mathcal{P}(P_{u} | P_1)=\frac{\mathcal{P}(P_{u} \land P_1)}{\mathcal{P}(P_1)} \quad \Longrightarrow \quad 
    \mathcal{P}(P_{u} \land P_1) &= \mathcal{P}(P_{u} | P_1).\mathcal{P}(P_1) \quad \text{and} \\ 
    \mathcal{P}(P_{u} \land P_0) &= \mathcal{P}(P_{u} | P_0).\mathcal{P}(P_0)\ ,
    \end{split}
    \label{}
    \end{align}
    where $\mathcal{P}(P_0)=\delta$ and $\mathcal{P}(P_1)=1-\delta$ based on the user-specified $\delta$ as the error level of true label coverage in CP method. Then, the following equation holds:
    \begin{equation}
        U_{\mathcal{C}}(x_{val}) = \mathcal{P}(P_{u} | P_1).\mathcal{P}(P_1) + \mathcal{P}(P_{u} | P_0).\mathcal{P}(P_0)\ .
    \label{uncer_conditional}
    \end{equation}
    If $P_0$ holds, it means $y_{val} \notin \mathcal{C}(x_{val})$. The prediction set without the true label does not yield to an acceptable  predictive outcomes. In this case, we can consider the model to be fully uncertain such that $\mathcal{P}(P_{u} | P_0)=1$. Otherwise, if $P_1$ holds, it means $y_{val} \in \mathcal{C}(x_{val})$. In this case, the true label is included in the prediction set and the pure model uncertainty $\widehat{u}_{\mathcal{C}}$ is an indicator of the baseline model uncertainty associated with the prediction set such that $\mathcal{P}(P_{u} | P_1) = \widehat{u}_{\mathcal{C}}$. We can now rewrite Equation~\ref{uncer_conditional} as,
    \begin{equation}
        U_{\mathcal{C}}(x_{val}) = \widehat{u}_{\mathcal{C}}.\mathcal{P}(P_1) + \mathcal{P}(P_0)\ .
    \label{uncer_prob}
    \end{equation}
    According to Theorem~\ref{cp_coverage_theorem}, the following upper and lower bounds hold for the probabilities $\mathcal{P}(P_1)$ and $\mathcal{P}(P_0)$ (i.e., $1-\mathcal{P}(P_1)$) to guarantee the true label coverage in the prediction set as,
    \begin{align}
        1-\delta \leq \mathcal{P}(P_1) &\leq 1-\delta+\frac{1}{n+1} \label{p1_bound} \quad \text{and} \\
        \delta-\frac{1}{n+1} &\leq \mathcal{P}(P_0) \leq \delta\  .
        \label{p0_bound}
    \end{align}
    Now, we can use the pure uncertainty $0 \leq \widehat{u}_{\mathcal{C}} \leq 1$, and the upper and lower bounds in Equations~\ref{p1_bound} and~\ref{p0_bound} to construct the bounds for the model uncertainty $U_{\mathcal{C}}(x_{val})$ based on Equation~\ref{uncer_prob} as,
    \begin{align}
        &\widehat{u}_{\mathcal{C}}.\mathcal{P}(P_1) + \mathcal{P}(P_0) \geq \widehat{u}_{\mathcal{C}}(1-\delta)+\delta-\frac{1}{n+1} \quad \text{and} \\ 
        &\widehat{u}_{\mathcal{C}}.\mathcal{P}(P_1) + \mathcal{P}(P_0) \leq \widehat{u}_{\mathcal{C}}(1-\delta+\frac{1}{n+1})+\delta\ .
    \end{align}
    Finally, we have:
    \begin{align}
    \begin{split}
        &U_{\mathcal{C}}(x_{val}) \geq \widehat{u}_{\mathcal{C}}(1-\delta)+\delta-\frac{1}{n+1}=\mathcal{L}_{\mathcal{C}} \quad \text{and} \\  
        &U_{\mathcal{C}}(x_{val}) \leq \widehat{u}_{\mathcal{C}}(\frac{n+2}{n+1})+\delta(1-\widehat{u}_{\mathcal{C}})=\mathcal{H}_{\mathcal{C}}\ ,
    \end{split}
    \label{case_2}
    \end{align}
    where $\widehat{u}_{\mathcal{C}}=\frac{m+\delta-1}{K}$, and $\mathcal{L}_{\mathcal{C}}$ and $\mathcal{H}_{\mathcal{C}}$ denote the conformal model uncertainty lower and upper bounds, respectively.

\begin{figure}[t] 
\centering
\subfloat[Possible class labels: $K=10$]{
    \includegraphics[width=\linewidth]{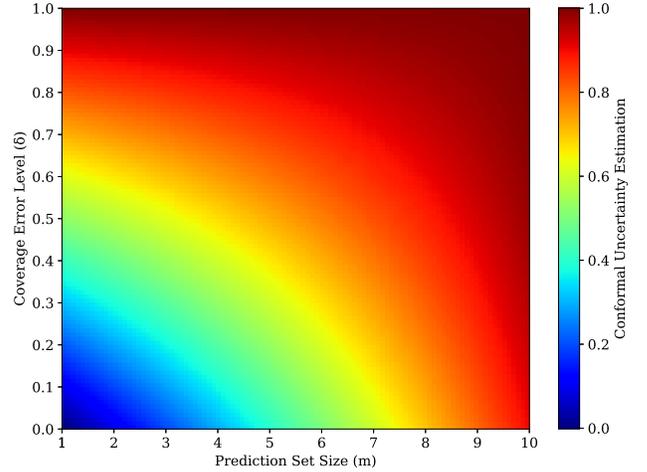}%
    \label{uncer_10_plot}
    }
\hfill
\subfloat[Possible class labels: $K=100$]{
    \includegraphics[width=\linewidth]{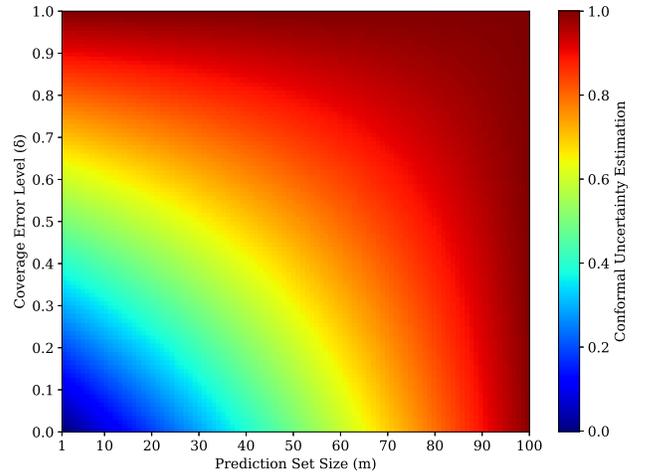}%
    \label{uncer_100_plot}
    }
\caption{The conformal model uncertainty associated with prediction sets of different sizes $m$ based on the variation of error level $\delta$}
\end{figure}

\begin{figure}[t] 
\centering
\subfloat[Possible class labels: $K=10$]{
    \includegraphics[width=\linewidth]{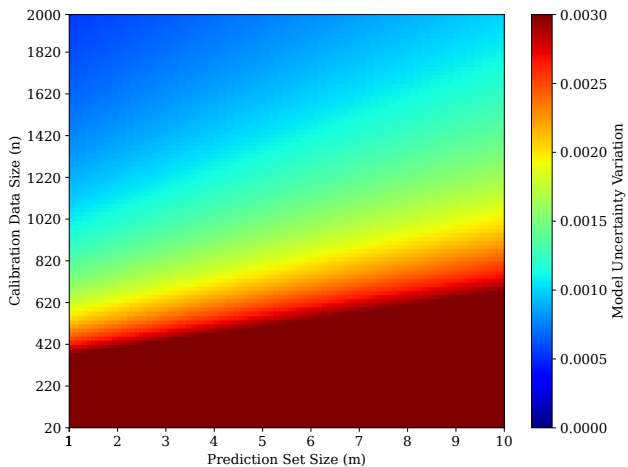}%
    \label{uncer_var_10_plot}
    }
\hfill
\subfloat[Possible class labels: $K=100$]{
    \includegraphics[width=\linewidth]{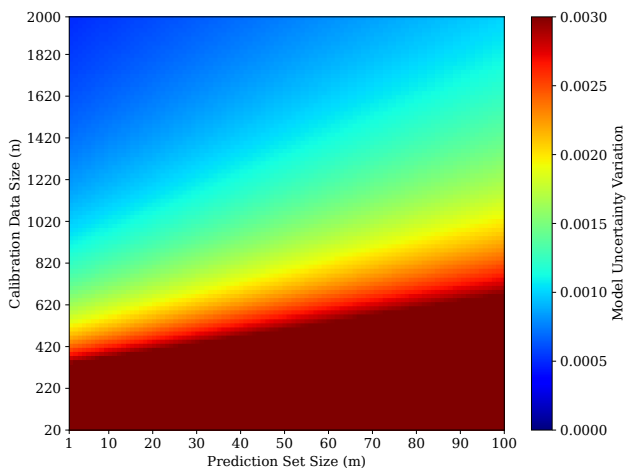}%
    \label{uncer_var_100_plot}
    }
\caption{The conformal model uncertainty variation $d_{\mathcal{C}}$ associated with prediction sets of different sizes $m$ based on the variation of calibration data size $n$ when $\delta$ is fixed}
\end{figure}

When $m=K$, we compute Equations~\ref{uncer_base} and~\ref{case_2} by setting $m$ to $K$, and now, we have all the possible labels in the prediction set representing that the model is highly uncertain and could not exclude any of the labels, i.e., could not make the set size smaller. Therefore, the upper bound $\mathcal{H}_{\mathcal{C}}$ of the conformal model uncertainty $U_{\mathcal{C}}(x_{val})$ is maximum and set to be $\min(\mathcal{H}_{\mathcal{C}},1)$.
% \begin{equation}
%     \min(1,U_{\mathcal{C}}(x_{val}))=\min(1,\frac{n+2}{n+1})=1\ .
% \end{equation}
% Then, we have:
% \begin{equation}
% 1-\frac{1}{n+1} \leq U_{\mathcal{C}}(x_{val}) \leq 1\ .
% \end{equation}
\end{proof}

In the following section, we discuss the validity and interpretations of the proposed uncertainty quantification method in conformal prediction.

\subsection{Interpretation of the Uncertainty Quantification}
According to the proposed Theorem~\ref{cp_uncer_theorem}, we can quantify the model uncertainty from the produced prediction sets in CP. This quantification represents the amount of uncertainty that a conformal model encounters when producing a prediction set of classifying an unseen validation data point. Based on the proposed Theorem~\ref{cp_uncer_theorem}, the model uncertainty is highly affected by two different measures: (1) the size of the produced prediction set, so that the larger set size indicates the higher model uncertainty associated with an unseen data point, and (2) the error level $\delta$ of true label coverage, so that the higher error level $\delta$ gives rise to lower value of $1-\delta$ which represents a lower probability of true label inclusion in the prediction set. When true label is not included in the prediction set, the model is expected to be highly uncertain in the prediction, therefore, the model shows higher uncertainty as the probability of true label inclusion in the prediction set is decreased. 

Figures~\ref{uncer_10_plot} and~\ref{uncer_100_plot} indicate the trend of the model uncertainty quantified for different sizes of an arbitrary prediction set based on the variation of coverage error level $\delta$ with the number of possible labels $K=10$ and $K=100$, respectively. For any arbitrary number of possible class labels, e.g., $K=10$ or $K=100$, we can obviously observe that the quantified conformal model uncertainty is consistently increasing with the growth in both prediction set size $m$ and the error level $\delta$ of true label coverage. The conformal model uncertainty is increased when the size of the prediction set is increasing which is an indicator of higher uncertainty. Furthermore, by increasing the error level $\delta$, the probability of true label inclusion in the prediction set is decreased and the model should become more uncertain in the prediction. Note that when $m=1$ and $\delta=0$, the produced prediction set contains only one label which is definitely the true label; therefore, the model has the minimum uncertainty on its prediction, i.e., maximum predictive confidence, which is the desired outcome.

According to the proposed Theorem~\ref{cp_uncer_theorem}, we quantify the upper bound $\mathcal{H}_{\mathcal{C}}$ and the lower bound $\mathcal{L}_{\mathcal{C}}$ for the conformal model uncertainty $U_{\mathcal{C}}(x_{val})$. There is a certified interval of uncertainty variations in the proposed quantification method denoted by $d_{\mathcal{C}}$ that is caused by the upper $\mathcal{H}_{\mathcal{C}}$ and the lower $\mathcal{L}_{\mathcal{C}}$ bounds. We can compute the magnitude of the model uncertainty variation as,
\begin{equation}
    d_{\mathcal{C}}=|\mathcal{H}_{\mathcal{C}} - \mathcal{L}_{\mathcal{C}}|=\frac{1+\widehat{u}_{\mathcal{C}}}{n+1}\ ,
\label{uncer_var_eq}
\end{equation}
where $\widehat{u}_{\mathcal{C}}=\frac{m+\delta-1}{K}$ is computed as the pure model uncertainty in Equation~\ref{uncer_base}, $K$ denotes the number of possible labels, $m$ denotes the prediction set size, $n$ denotes the size of calibration data set, and $\delta$ denotes the coverage error level of the true label.
This magnitude of the uncertainty variations indicates the tightness and the accuracy of the proposed uncertainty estimation. Higher $d_{\mathcal{C}}$ represents a larger variation interval of the model uncertainty.  
Figures~\ref{uncer_var_10_plot} and~\ref{uncer_var_100_plot} demonstrate the amount of the uncertainty variation $d_{\mathcal{C}}$ based on the calibration set size $n$ and the prediction set size $m$ for the number of possible class labels $K=10$ and $K=100$, respectively. We can observe that for a fixed amount of calibration set size, when the prediction set size is increased, $d_{\mathcal{C}}$ as the magnitude of uncertainty variations is increased as well. This observation shows that a smaller prediction set yields to a tighter certified bound for the conformal model uncertainty which represents a more accurate estimation of uncertainty. Moreover, when the calibration set size is increased, the magnitude of uncertainty variation interval $d_{\mathcal{C}}$ is significantly decreased in order to provide a tighter bound of uncertainty estimation since by having larger set of calibration data, the model can compute more accurate $\widehat{q}$ as the $1-\delta$ quantile of the conformal scores in the calibration data.

% \begin{figure}[t] 
% \centering
% \includegraphics[width=\linewidth]{}
% \caption{The empirical CDF of entropy for OOD data}
% \label{}
% \end{figure}

\section{Conclusion}
\label{conc}
In this paper, we have addressed the problem of model uncertainty quantification in conformal prediction. Through our investigation, we proposed a novel technique to enhance the reliability and accuracy of uncertainty estimation. We used the existing statistical guarantee of the true label coverage in the prediction sets to quantify the model uncertainty in a probabilistic view, and certify upper and lower bounds for the uncertainty quantification. 
Our findings highlight the important implications of accurate uncertainty quantification, representing its benefits for decision-making and risk assessment in real-world applications. 

While our research has made notable contributions, there are still opportunities for further exploration. Future work should focus on addressing challenges such as high-dimensional data, imbalanced datasets, and incorporating domain knowledge into uncertainty quantification in conformal prediction. Additionally, investigating interpretability and explainability of uncertainty measures can provide actionable insights. We encourage continued research to foster the development of more reliable and accurate uncertainty quantification methods within the conformal prediction framework.

% In summary, this paper advances model uncertainty quantification in conformal prediction, offering an improved technique to compare CP-based methods with other state-of-the-art methods. 

% \appendix

% \section{Acknowledgments}

% \bigskip
% \noindent Thank you for reading these instructions carefully. We look forward to receiving your electronic files!

\bibliography{uncer_cp}

\end{document}